\documentclass[final,12pt]{alt2022-plain}

\usepackage{hyperref}

\usepackage{amsfonts}
\usepackage{amsopn}
\usepackage[mathscr]{eucal}
\usepackage{color}

\hypersetup{
pageanchor=true,
plainpages=false,
pdfpagelabels=true,
bookmarks=true,
bookmarksnumbered=true,
breaklinks=true,
pdfborder={0 0 0}
}

\usepackage{url}

\usepackage{ulem}
\renewenvironment{proof}[1][]{\par\noindent{\bf Proof #1\ }}{\hfill\BlackBox\\[2mm]}

\newcommand{\X}{\mathcal X}

\newcommand{\Y}{\mathcal Y}
\newcommand{\F}{\mathcal F}

\renewcommand{\L}{\mathcal L}

\newcommand{\T}{\mathcal T}

\renewcommand{\P}{\mathbb P}

\newcommand{\nats}{\mathbb{N}}

\newcommand{\reals}{\mathbb{R}}

\newcommand{\E}{\mathbb E}

\renewcommand{\limsup}{\mathop{\rm limsup}}

\DeclareSymbolFont{bbold}{U}{bbold}{m}{n}
\DeclareSymbolFontAlphabet{\mathbbold}{bbold}
\newcommand{\ind}{\mathbbold{1}}

\newcommand{\Borel}{{\cal B}}

\newcommand{\ProcX}{\mathbb{X}}
\newcommand{\ProcY}{\mathbb{Y}}

\newcommand{\target}{f^{\star}}
\newcommand{\goodfun}{f_0}
\newcommand{\loss}{\ell}

\newcommand{\SUOL}{{\rm SUOL}}

\newcommand{\ProcSet}{\mathcal{C}}

\newcommand{\ignore}[1]{}
\newcommand{\private}[1]{}

\makeatletter
\newcommand{\vast}{\bBigg@{3}}
\newcommand{\Vast}{\bBigg@{4}}
\makeatother

\newsavebox{\savepar}

\newtheorem{condition}{Condition}
\newtheorem{problem}{Open Problem}

\title[Universal Consistency with Dependent Responses]{Universally Consistent Online Learning with \\Arbitrarily Dependent Responses}
\usepackage{times}

\altauthor{%
 \Name{Steve Hanneke} \Email{steve.hanneke@gmail.com}\\
 \addr Purdue University%
}

\begin{document}

\newlength{\pspace}
\setlength{\pspace}{2mm}

\maketitle

\begin{abstract}%
  This work provides an online learning rule that is universally consistent
  under processes on $(X,Y)$ pairs, under conditions only on the $X$ process.
  As a special case, the conditions admit all processes on $(X,Y)$ such that the process on $X$ is stationary.
  This generalizes past results which required stationarity for the joint process on $(X,Y)$, and additionally
  required this process to be ergodic.  In particular, this means that ergodicity is superfluous for the purpose of
  universally consistent online learning.
\end{abstract}

\begin{keywords}
statistical learning theory,
universal consistency,
nonparametric estimation,
stochastic processes,
stationary processes, 
non-ergodic processes,
online learning
\end{keywords}

\section{Introduction}
\label{sec:intro}

The task of achieving low expected \emph{regret} in online learning is a classic topic in learning theory.  
Specifically, we consider a sequential setting, where at each time $t$, a learner observes a point $X_t$, 
makes a \emph{prediction} $\hat{Y}_t$, and then observes a true \emph{response} $Y_t$: that is, 
$\hat{Y}_t = f_t(X_{1:(t-1)},Y_{1:(t-1)},X_t)$ for some function $f_t$ (possibly randomized).
We are then interested in the rate of growth of the long-run cumulative \emph{loss} of the learner: i.e., 
$\sum_{t=1}^{T} \loss(\hat{Y}_t,Y_t)$, for a given loss function $\loss$.
However, as it may sometimes be impossible to achieve low cumulative loss in an absolute sense, 
we are often interested in understanding the \emph{excess} loss compared to some particular 
\emph{fixed} predictor $\goodfun$: 
i.e., $\sum_{t=1}^{T} \loss(\hat{Y}_t,Y_t) - \sum_{t=1}^{T} \loss(\goodfun(X_t),Y_t)$, 
known as the \emph{regret} (relative to $\goodfun$).

{\vskip \pspace}Several different formulations of the subject have been proposed, leading to different algorithmic approaches 
and theoretical analyses of regret.  For instance, there is a rich theory of online learning with \emph{arbitrary} 
sequences $\{(X_t,Y_t)\}_{t=1}^{\infty}$, but where the reference function $\goodfun$ is restricted to belong to some 
particular function class $\F$ \citep*[see e.g.,][]{cesa-bianchi:06,ben-david:09,rakhlin:15}.  

{\vskip \pspace}On the other hand, 
there has also been significant work on theories that allow \emph{any} reference function $\goodfun$, 
while restricting the sequence $\{(X_t,Y_t)\}_{t=1}^{\infty}$.  This is the subject of the present work.
A classic starting point for this line of work is the theory of \emph{universally consistent} predictors 
under \emph{i.i.d.}\ sequences.  In particular, for binary classification and bounded regression, 
under mild conditions on the value space $\X$ of the $X_t$ variables,  
there are simple learning rules $\hat{f}_t$ satisfying a guarantee that, for any i.i.d.\ sequence $\{(X_t,Y_t)\}_{t=1}^{\infty}$, 
for every measurable function $\goodfun : \X \to \Y$, it holds that   
$\sum_{t=1}^{T} \loss(\hat{Y}_t,Y_t) - \sum_{t=1}^{T} \loss(\goodfun(X_t),Y_t) = o(T)$ almost surely
\citep*[e.g.,][]{stone:77,devroye:96,gyorfi:02,hanneke:21b}.
This result has since been extended to various non-i.i.d.\ conditions on $\{(X_t,Y_t)\}_{t=1}^{\infty}$, 
including the case where $\{(X_t,Y_t)\}_{t=1}^{\infty}$ is any \emph{stationary ergodic} process 
\citep*[e.g.,][]{morvai:96,gyorfi:99,gyorfi:02}, or generally satisfies a 
\emph{law of large numbers} \citep*{morvai:99,steinwart:09}.

{\vskip \pspace}In the present work, we are interested in exploring \emph{weakened} assumptions on $\{(X_t,Y_t)\}_{t=1}^{\infty}$ 
such that there still exist learners that are consistent for all $\{(X_t,Y_t)\}_{t=1}^{\infty}$ satisfying the assumption.

{\vskip \pspace}The recent work of \citep*{hanneke:21} focuses on the input sequence $\{X_t\}_{t=1}^{\infty}$, 
investigating the \emph{minimal} assumption for this sequence to admit consistent learners 
under the restriction that $Y_t = \target(X_t)$ for an arbitrary fixed function $\target$ 
(or in some cases, more generally with random variables $Y_t$, where $\E[\loss(\target(X_t),Y_t)|X_t]$ is minimal, and 
the $Y_t$ sequence is conditionally independent given the respective $X_t$ variables).
While that work identifies such provably-minimal assumptions for related settings (namely, inductive and self-adaptive learning), 
for the setting of online learning it only establishes a necessary condition and a sufficient condition which 
are provably distinct, leaving open the problem of exactly identifying the precise minimal condition admitting 
consistent learners \citep*{hanneke:21c}.

{\vskip \pspace}Nevertheless, the sufficient condition (Condition~\ref{con:kc} below) established by \citep*{hanneke:21} is 
quite general, encompassing all previously-studied conditions on $\{X_t\}_{t=1}^{\infty}$ admitting universal consistency.
It is therefore interesting to investigate whether this condition remains sufficient for universal consistency \emph{without} restricting to 
deterministic responses $Y_t =\target(X_t)$ (or conditional independence of $Y_t$ values).  
In particular, it is most interesting to understand whether there is a 
family of processes $\{(X_t,Y_t)\}_{t=1}^{\infty}$ admitting universally consistent learners, 
and encompassing \emph{all} previously-studied families admitting 
such consistent learners.

{\vskip \pspace}In this work, we propose such a family.  Specifically, the family consists of all $\{(X_t,Y_t)\}_{t=1}^{\infty}$ where 
$\{X_t\}_{t=1}^{\infty}$ satisfies the condition suggested by \citep*{hanneke:21} (Condition~\ref{con:kc} below).  
In particular, within this family, the sequence $\{Y_t\}_{t=1}^{\infty}$ is \emph{completely unrestricted}.  
We establish this result for all bounded separable metric losses $\loss$.

{\vskip \pspace}As an interesting implication, while past works established that there are universally consistent learners 
for the family of all stationary ergodic processes \citep*{morvai:96,gyorfi:99,gyorfi:02}, 
here we find that the requirement of ergodicity is completely \emph{superfluous}: that is, 
\emph{stationarity alone} is already sufficient for universally consistent online learning.
Indeed, it even suffices if only the \emph{input} sequence $\{X_t\}_{t=1}^{\infty}$ is stationary, 
while the response sequence $\{Y_t\}_{t=1}^{\infty}$ can be essentially \emph{arbitrary}.
As discussed below, the proposed family of processes also encompasses interesting families of \emph{non-stationary} 
processes previously studied in the literature \citep*[e.g.,][]{ryabko:06}.

{\vskip \pspace}The algorithm achieving this result 
essentially uses an infinite variant of the Hedge algorithm for learning with expert advice, 
applied to a countable set of functions, which were shown by \citep*{hanneke:21} to be \emph{dense}, 
in an appropriate sense defined relative to long-run averages observed in any process $\ProcX$ satisfying 
a key condition (Condition~\ref{con:kc} below).
A related approach has previously been shown to be consistent under stationary ergodic processes \citep*{gyorfi:99}.

\section{Formal Setup and Main Result}
\label{sec:main}

Following the setup from \citep*{hanneke:21}, 
we suppose $\X$ is a non-empty set, equipped with a separable metrizable topology $\T$,
and $\Borel$ denotes the Borel $\sigma$-algebra generated by $\T$, defining the measurable subsets of $\X$.
We allow $(\Y,\loss)$ to be any non-empty bounded separable metric space, where for simplicity we suppose $\sup_{y,y' \in \Y} \loss(y,y') \leq 1$.
For instance, this covers the classification setting, where $\Y$ is any non-empty set and $\loss(y,y') = \ind[ y \neq y' ]$, 
as well as regression on $\Y = [0,1]$ with the absolute loss $\loss(y,y') = |y-y'|$.

{\vskip \pspace}We will be interested in stochastic processes $(\ProcX,\ProcY) := \{ (X_t,Y_t) \}_{t=1}^{\infty}$ where each $(X_t,Y_t)$ takes values in $\X \times \Y$.
Following \citep*{hanneke:21}, for any measurable $A \subseteq \X$, define
$\hat{\mu}_{\ProcX}(A) = \limsup\limits_{n \to \infty} \frac{1}{n} \sum_{t=1}^{n} \ind_{A}(X_t)$.
Then we consider the following condition from \citep*{hanneke:21}.

\begin{condition}[\citealp*{hanneke:21}]
  \label{con:kc}
  For every monotone sequence $\{A_k\}_{k=1}^{\infty}$ of measurable subsets of $\X$ with $A_k \downarrow \emptyset$,
  \begin{equation*}
    \lim\limits_{k \to \infty} \E\!\left[ \hat{\mu}_{\ProcX}(A_k) \right] = 0.
  \end{equation*}
\end{condition}

As discussed by \citep*{hanneke:21}, this condition stipulates that $\E\!\left[\hat{\mu}_{\ProcX}(\cdot)\right]$ should behave as a \emph{continuous submeasure}.
For instance, \citep*{hanneke:21} shows that every process $\ProcX$ with convergent relative frequencies satisfies Condition~\ref{con:kc},
which includes all \emph{stationary} processes, or even asymptotically mean stationary processes; indeed, in these cases, 
Birkhoff's ergodic theorem implies that $\E\!\left[\hat{\mu}_{\ProcX}(\cdot)\right]$ behaves as a \emph{probability measure}, and hence is always continuous in the above sense.
\citep*{hanneke:21} also argues that many interesting completely non-stationary processes satsify the condition as well.
Moreover, note that ergodicity is not needed for this to hold.
Condition~\ref{con:kc} will arise in our analysis as supplying the existence of a countable set $\tilde{\F}$ of functions 
which are \emph{dense} in an appropriate sense.

{\vskip \pspace}We consider \emph{online learning rules}: that is, sequences of measurable functions $f_t : \X^{t-1} \times \Y^{t-1} \times \X \to \Y$, 
which in general may be \emph{randomized} (where the internal randomness should be independent from the data sequence). 
In this general setting, for any measurable $\goodfun : \X \to \Y$,
for any process $(\ProcX,\ProcY) = \{(X_{t},Y_{t})\}_{t=1}^{\infty}$ on $\X \times \Y$,
for any online learning rule $f_t$, 
we define the long-run average excess loss
\begin{equation*}
  \hat{\L}_{(\ProcX,\ProcY)}(f_{\cdot};\goodfun) = \limsup\limits_{T \to \infty} \frac{1}{T} \sum\limits_{t=1}^{n} \left( \loss\!\left( f_{t}(X_{1:(t-1)}, Y_{1:(t-1)}, X_{t}), Y_{t} \right) - \loss\!\left(\goodfun(X_{t}),Y_{t}\right) \right).
\end{equation*}
We are then interested in online learning rules $f_{t}$ satisfying the guarantee that, for all measurable functions $\goodfun : \X \to \Y$,
it holds that $\hat{\L}_{(\ProcX,\ProcY)}(f_{\cdot};\goodfun) \leq 0$ almost surely,
for every $(\ProcX,\ProcY)$ in some specific family of processes $\ProcSet$.
Such an online learning rule $f_t$ is said to be \emph{strongly universally consistent} for $\ProcSet$.
The main result is the following.

\begin{theorem}
\label{thm:online-kc}
There is an online learning rule that is strongly universally consistent for the set of all processes $(\ProcX,\ProcY)$ such that $\ProcX$ satisfies Condition~\ref{con:kc}.
\end{theorem}

In particular, this has the following immediate corollary.

\begin{corollary}
  \label{cor:online-stationary}
 There is an online learning rule that is strongly universally consistent for the set of all processes $(\ProcX,\ProcY)$ such that $\ProcX$ is stationary.
\end{corollary}

\paragraph{Relation to prior results} 
As discussed above, past works have established universal consistency for the set of all processes $(\ProcX,\ProcY)$ that are stationary \emph{and ergodic}
\citep*[e.g.,][]{morvai:99,gyorfi:99,gyorfi:02}, or otherwise satisfy a law of large numbers \citep*{morvai:99,steinwart:09}.  
The work of \citet*{ryabko:06} established universal consistency of certain learning rules for classification,
under a family of processes that allow the $Y_t$ sequence to be \emph{arbitrary}, but restrict the $X_t$ sequence to be 
conditionally i.i.d.\ given a $Y_t$ value, and respecting that $Y_t$ is determined by $X_t$.
\citep*{hanneke:21} has shown that any process $\ProcX = \{X_t\}_{t=1}^{\infty}$ satisfying the condition from \citet*{ryabko:06} 
also satisfies Condition~\ref{con:kc}.
\citep*{hanneke:21} established universal consistency for $\ProcX$ satisfying Condition~\ref{con:kc},
when $Y_t = \target(X_t)$ for a fixed (arbitrary, unknown) measurable function $\target : \X \to \Y$ 
(or more-generally, when $\E[\loss(\target(X_t),Y_t)|X_t]$ is minimal, and the $Y_t$ variables are conditionally independent 
given their respective $X_t$ variables).
We note that Theorem~\ref{thm:online-kc} unifies and generalizes all of these prior results on families of processes admitting universally consistent learners.
Another interesting work is that of \citet*{kulkarni:02}, 
which restricts the $Y_t$ sequence to be conditionally independent with \emph{continuous} conditional mean $\target(X_t)$, 
but allows the $X_t$ sequence to be \emph{arbitrary}.  They establish consistency of certain learning rules under these conditions.
That work is, in a certain sense, \emph{dual} to the present work, as we seek restrictions on the $X_t$ sequence while allowing the $Y_t$ 
sequence to be arbitrary.  As such, the results cannot be directly compared.  
See \citep*{hanneke:21} for a thorough summary of past work on universal consistency for general families of stochastic processes.

\section{Proof of the Theorem}
\label{sec:proof}

The essential approach is to apply an infinite variant of the Hedge algorithm for learning with expert advice, 
where the experts are given by a well-chosen countable set of measurable functions which are ``dense'' in an appropriate sense.
This is similar in spirit to certain existing strategies known to be consistent under stationary ergodic processes \citep*{gyorfi:99}.

{\vskip \pspace}Before proving Theorem~\ref{thm:online-kc}, we introduce some useful results from the literature.
We first state a well-known result from the literature on prediction with expert advice 
\citep*{cesa-bianchi:06,hannan:57,vovk:90,vovk:92,littlestone:94,freund:97b,cesa-bianchi:97,kivinen:99,singer:99,gyorfi:02a}.
See Corollary~4.2 of \citep*{cesa-bianchi:06} for this specific result.

\begin{lemma}
\label{lem:hedge}
Fix $N,n \in \nats$.
Let $y_1,y_2,\ldots$ be a sequence in $\Y$.
For each $t \in \{1,\ldots,n\}$, let $\{z_{t,i}\}_{i=1}^{N}$ be a sequence of values in $\Y$.
Let $\eta = \sqrt{(8/n)\ln(N)}$.
For each $t \in \nats$ and $i \in \{1,\ldots,N\}$,
define $L_{t,i} = \frac{1}{t} \sum\limits_{s=1}^{t} \loss(z_{s,i},y_s)$.
Then for each $i \in \{1,\ldots,N\}$, define $w_{1,i} = v_{1,i} = 1/N$,
and for each $t \in \{2,\ldots,n\}$, define
$w_{t,i} = (1/N) e^{- \eta (t-1) L_{(t-1),i} }$,
and $v_{t,i} = w_{t,i} / \sum\limits_{j=1}^{N} w_{t,j}$.
Finally, let $\{ \hat{z}_{t} \}_{t \in \nats}$ be independent $\Y$-valued random variables, 
with $\P( \hat{z}_{t} = z_{t,i} ) = v_{t,i}$ for each $i \in \{1,\ldots,N\}$ and $t \in \nats$: 
that is, $\hat{z}_{t}$ is a random sample from the $v_{t,i}$-weighted distribution on $z_{t,i}$ values.
Then, for any fixed $\delta \in (0,1)$, 
with probability at least $1-\delta$, it holds that 
\begin{equation*}
\sum_{t=1}^{n} \loss(\hat{z}_{t},y_t) \leq \min_{i \in \{1,\ldots,N\}} n L_{n,i} + \sqrt{(1/2) n \ln(N)} + \sqrt{(1/2) n \ln\!\left(\frac{1}{\delta}\right)}.
\end{equation*}
\end{lemma}

{\vskip 1mm}\noindent In particular, this has the following implication for learning with an infinite number of experts.

\begin{corollary}
\label{cor:infinite-hedge}
Let $y_1,y_2,\ldots$ be a sequence in $\Y$ (possibly stochastic).
For each $t \in \nats$, let $\{z_{t,i}\}_{i=1}^{\infty}$ be a sequence of values in $\Y$ (possibly stochastic).
Let $t_1 = 1$ and $T_1 = \{1\}$, and inductively define $t_{j+1} = t_j + j$ and $T_{j+1} = \{t_{j+1},\ldots,t_{j+1} + j\}$ for each $j \in \nats$.
Also for each $j \in \nats$, define $\eta_j = \sqrt{(8 / j) \ln(j)}$.
For each $j \in \nats$ and $t \in T_j$, for $i \in \{1,\ldots,j\}$ define $L_{t,i} = \frac{1}{t-t_j+1} \sum\limits_{s=t_j}^{t} \loss(z_{s,i},y_s)$.
For each $i \in \{1,\ldots,j\}$, define $w_{t_j,i} = v_{t_j,i} = 1/j$,
and for each $t \in T_j \setminus \{t_j\}$, define $w_{t,i} = (1/j) e^{-\eta_j (t-t_j) L_{t-1,i}}$,
and $v_{t,i} = w_{t,i} / \sum\limits_{i' = 1}^{j} w_{t,i'}$.
Finally, let $\{ \hat{z}_{t} \}_{t \in \nats}$ be $\Y$-valued random variables such that $\hat{z}_t$ is conditionally independent of $\{ \hat{z}_{t'} \}_{t' \neq t}$ given $\{y_{t'}\}_{t' < t}$ and $\{z_{t',i}\}_{i \in \nats, t' \leq t}$, 
and for each $j \in \nats$ and $t \in T_j$, $\P( \hat{z}_{t} = z_{t,i} | \{y_{t'}\}_{t' < t}, \{z_{t',i}\}_{i \in \nats, t' \leq t}) = v_{t,i}$ for each $i \in \{1,\ldots,j\}$: 
that is, $\hat{z}_{t}$ is a random sample from the $v_{t,i}$-weighted distribution on $\{z_{t,1},\ldots,z_{t,j}\}$.
Then, with probability one, $\exists \hat{n} \in \nats$ such that every $n \in \nats$ with $n > \hat{n}$ satisfies 
\begin{equation*}
\sum_{t=1}^{n} \loss(\hat{z}_{t},y_t) \leq \left( \min_{1 \leq i \leq n^{1/4}} \sum_{t = 1}^{n} \loss(z_{t,i},y_{t}) \right) + 19\, n^{3/4} \sqrt{\ln(n)} + \hat{n}.
\end{equation*}
\end{corollary}
\begin{proof}
Applying Lemma~\ref{lem:hedge} under the conditional distribution given $\{y_{t}\}_{t \in \nats}$ and $\{z_{t,i}\}_{t,i \in \nats}$, 
for every $j \in \nats$, with conditional probability at least $1-\frac{1}{j^2}$, 
\begin{equation*}
\sum_{t \in T_j} \loss(\hat{z}_{t}, y_t) \leq \left( \min_{i \in \{1,\ldots,j\}} \sum_{t \in T_j} \loss(z_{t,i},y_{t}) \right) + \left(1 + \frac{1}{\sqrt{2}}\right) \sqrt{j \ln(j)}.
\end{equation*}
By the law of total probability, this holds with (unconditional) probability at least $1-\frac{1}{j^2}$ as well.
Since $\sum_{j \in \nats} \frac{1}{j^2} < \infty$, the Borel-Cantelli lemma implies that, with probability one, $\exists \hat{j} \in \nats$ such that 
the above inequality holds for every $j \in \nats$ with $j > \hat{j}$. 
For the remainder of the proof, we suppose this event occurs.
Without loss of generality, we may suppose $\hat{j} \geq 2$.

For any $n \in \nats$, let $j_n$ be the index such that $n \in T_{j_n}$: namely, $j_n = \left\lceil \frac{1}{2}\sqrt{8 n + 1} - \frac{1}{2} \right\rceil$.
Define $\hat{n} = \max T_{\hat{j}}$.
Then for any $n \in \nats$ with $n > \hat{n}$, letting $m(n) = \left\lceil \sqrt{n} \right\rceil$, we have 
\begin{align*}
  & \sum_{t=1}^{n} \loss(\hat{z}_{t},y_t) \leq \sum_{j = 1}^{j_n} \sum_{t \in T_j} \loss(\hat{z}_{t},y_t) 
\\ & \leq  \hat{n} + m(n) + j_{m(n)} + \sum_{j = j_{m(n)}+1}^{j_n} \left( \left( \min_{i \in \{1,\ldots,j\}} \sum_{t \in T_j} \loss(z_{t,i},y_{t}) \right) + \left( 1 + \frac{1}{\sqrt{2}} \right) \sqrt{j \ln(j)} \right)
\\ & \leq  \hat{n} + m(n) + j_{m(n)} + \min_{i \in \{1,\ldots,j_{m(n)}+1\}} \sum_{j = j_{m(n)}+1}^{j_n} \left( \left( \sum_{t \in T_j} \loss(z_{t,i},y_{t}) \right)  + \left( 1 + \frac{1}{\sqrt{2}} \right) \sqrt{j \ln(j)} \right).
\end{align*}
We may then note that 
\begin{equation*}
\sum_{j = j_{m(n)}+1}^{j_n} \sum_{t \in T_j} \loss(z_{t,i},y_{t}) 
\leq  \sum_{t = 1}^{t_{j_n+1}-1} \loss(z_{t,i},y_{t}) 
\leq j_n + \sum_{t = 1}^{n} \loss(z_{t,i},y_{t}) 
\end{equation*}
and 
\begin{align*}
& \sum_{j = j_{m(n)}+1}^{j_n} \left( 1 + \frac{1}{\sqrt{2}} \right) \sqrt{j \ln(j)} 
\leq \left( 1 + \frac{1}{\sqrt{2}} \right) \sqrt{\ln(j_n)} \sum_{j = 1}^{j_n} \sqrt{j}
\\ & \leq  \left( 1 + \frac{1}{\sqrt{2}} \right) \sqrt{\ln(j_n)} \int_{1}^{j_{n}+1} \sqrt{x} {\rm d}x
\leq  \left( 1 + \frac{1}{\sqrt{2}} \right) \sqrt{\ln(j_n)} \frac{2}{3} (j_{n}+1)^{3/2}.
\end{align*}
Altogether we have that 
\begin{align*}
\sum_{t=1}^{n} \loss(\hat{z}_{t},y_t) & \leq 
\hat{n} + m(n) + j_{m(n)} + \left( \min_{i \in \{1,\ldots,j_{m(n)}+1\}} \sum_{t = 1}^{n} \loss(z_{t,i},y_{t}) \right) 
\\ & {\hskip 11mm}+ j_{n} + \left( 1 + \frac{1}{\sqrt{2}} \right) \sqrt{\ln(j_n)} \frac{2}{3} (j_{n}\!+\!1)^{3/2}.
\end{align*}
Noting that $j_{n} \leq \sqrt{6 n}$ and $n^{1/4}-1 \leq j_{m(n)} \leq 4 n^{1/4}$,
we have 
\begin{align*}
\sum_{t=1}^{n} \loss(\hat{z}_{t},y_t) & \leq 
\hat{n} + 1 + \sqrt{n} + 4 n^{1/4} + \left( \min_{1 \leq i \leq n^{1/4}} \sum_{t = 1}^{n} \loss(z_{t,i},y_{t}) \right) 
\\ & {\hskip 11mm}+ \sqrt{6n} + \left( 1 + \frac{1}{\sqrt{2}} \right) \sqrt{(1/2) \ln(6n)} \frac{2}{3} (\sqrt{6n}+1)^{3/2}
\\ & \leq \left( \min_{1 \leq i \leq n^{1/4}} \sum_{t = 1}^{n} \loss(z_{t,i},y_{t}) \right) + 19\, n^{3/4} \sqrt{\ln(n)} + \hat{n}.
\end{align*}
\end{proof}

We will also use the following result, which is Lemma 35 from \citep*{hanneke:21}.

\begin{lemma}
\label{lem:array-convergence}
\citep*[][Lemma 35]{hanneke:21}~
Suppose $\{\beta_{k,n}\}_{k,n \in \nats}$ is an array of values in $[0,\infty)$ 
such that $\lim\limits_{k \to \infty} \limsup\limits_{n \to \infty} \beta_{k,n} = 0$, 
and that $\{k_n\}_{n \in \nats}$ is a sequence in $\nats$ with $k_n \to \infty$.
Then there exists a sequence $\{j_n\}_{n \in \nats}$ in $\nats$ such that $j_n \leq k_n$ for every $n \in \nats$, 
and $\lim\limits_{n \to \infty} \beta_{j_n,n} = 0$.
\end{lemma}

Finally, we will make use of one additional result. 
Following \citep*{hanneke:21}, for a process $\ProcX = \{X_t\}_{t=1}^{\infty}$ and a measurable function $g : \X \to \reals$,
define 
\begin{equation*}
\hat{\mu}_{\ProcX}(g) := \limsup\limits_{m \to \infty} \frac{1}{m} \sum_{t=1}^{m} g(X_t).
\end{equation*}
\citep*{hanneke:21} proves the following result.

\begin{lemma}
\label{lem:countable-dense}
\citep*[][Lemma 24]{hanneke:21}~
There exists a countable set $\tilde{\F}$ of measurable functions $\X \to \Y$ such that,
for every measurable function $f : \X \to \Y$, for every process $\ProcX$ satisfying Condition~\ref{con:kc}, 
\begin{equation*}
\inf_{\tilde{f} \in \tilde{\F}} \E\!\left[ \hat{\mu}_{\ProcX}\!\left( \loss\!\left(\tilde{f}(\cdot),f(\cdot) \right) \right) \right] = 0.
\end{equation*}
\end{lemma}

We are now ready for the proof of Theorem~\ref{thm:online-kc}.

\begin{proof}[of Theorem~\ref{thm:online-kc}]
Let $\tilde{\F}$ be as in Lemma~\ref{lem:countable-dense}, and enumerate $\tilde{\F} = \{\tilde{f}_1,\tilde{f}_2,\ldots\}$.

  Define the online learning rule $\hat{f}_t$ as follows.
  For any sequences $x_{1:t} \in \X^t$ and $y_{1:(t-1)} \in \Y^{t-1}$, 
  let $z_{s,i} = \tilde{f}_i(x_s)$ for each $s \leq t$ and $i \in \nats$, and 
  let $\hat{z}_{t}$ be defined as in Corollary~\ref{cor:infinite-hedge} 
  (for these $z_{s,i}$ and $y_s$ values).  Note that $\hat{z}_{t}$ is defined 
  purely in terms of these values and internal randomness of the learner (i.e., the rest of the infinite sequence of $y_s$ values ($s \geq t$) and array of $z_{s,i}$ values ($s > t$) 
  needn't be defined for the purpose of defining $\hat{z}_t$). 
  Finally, define $\hat{f}_{t}(x_{1:(t-1)},y_{1:(t-1)},x_t) := \hat{z}_t$ in this context.

  Fix any process 
  $(\ProcX,\ProcY) = \{(X_t,Y_t)\}_{t=1}^{\infty}$
  with $\ProcX$ satisfying Condition~\ref{con:kc}, 
  and fix any measurable function $\goodfun : \X \to \Y$.
  Let $i_{k}$ be a sequence in $\nats$ such that 
\begin{equation*}
\E\!\left[ \hat{\mu}_{\ProcX}\!\left( \loss\!\left(\tilde{f}_{i_k}(\cdot),\goodfun(\cdot) \right) \right) \right] < 2^{-2k}
\end{equation*}
for all $k \in \nats$, guaranteed to exist by the defining property of $\tilde{\F}$ from Lemma~\ref{lem:countable-dense}.
By Markov's inequality, for each $k$, with probability at least $1-2^{-k}$, 
\begin{equation}
  \label{eqn:approx-bound}
  \hat{\mu}_{\ProcX}\!\left( \loss\!\left(\tilde{f}_{i_k}(\cdot), \goodfun(\cdot) \right) \right) < 2^{-k}.
\end{equation}
Thus, since $\sum_{k \in \nats} 2^{-k} < \infty$, by the Borel-Cantelli lemma, on an event $E_0$ of probability one, 
there exists $\kappa \in \nats$ such that, for all $k \geq \kappa$, \eqref{eqn:approx-bound} holds.
For each $n \in \nats$, define 
\begin{equation*}
k_n = \max\{ k : k \leq n, \max\{i_1,\ldots,i_k\} \leq n^{1/4} \},
\end{equation*} 
and note that $k_n \to \infty$. 

By Corollary~\ref{cor:infinite-hedge} and the definition of $\hat{f}_{t}$, 
there is an event $E_1$ of probability one, on which $\exists \hat{n} \in \nats$ such that every $n \in \nats$ with $n > \hat{n}$ satisfies 
\begin{equation*}
\frac{1}{n} \!\sum_{t=1}^{n} \loss\!\left(\hat{f}_{t}(X_{1:(t-1)},Y_{1:(t-1)},X_t),Y_t\right) \!\leq \!\left( \min_{1 \leq i \leq n^{1/4}} \frac{1}{n} \!\sum_{t = 1}^{n} \loss\!\left(\tilde{f}_{i}(X_{t}),Y_{t}\right) \right) + 19\, n^{-1/4} \!\sqrt{\log(n)} + \frac{\hat{n}}{n}.
\end{equation*}
Since $19\, n^{-1/4} \sqrt{\log(n)} \to 0$ and $\frac{\hat{n}}{n} \to 0$ as $n \to \infty$, it suffices to focus on the first term on the right hand side.
For this term, by the triangle inequality, we have 
\begin{equation*}
\min_{1 \leq i \leq n^{1/4}} \frac{1}{n} \sum_{t = 1}^{n} \loss\!\left(\tilde{f}_{i}(X_{t}),Y_{t}\right) 
\leq \frac{1}{n} \sum_{t = 1}^{n} \loss\!\left(\goodfun(X_{t}),Y_t\right)  ~+ \min_{1 \leq i \leq n^{1/4}} \frac{1}{n} \sum_{t = 1}^{n} \loss\!\left(\tilde{f}_{i}(X_{t}),\goodfun(X_{t})\right).
\end{equation*}
Thus, to complete the proof it suffices to argue that the second term on the right hand side approaches zero almost surely as $n \to \infty$.

For each $k,n \in \nats$, let $\beta_{k,n} = \frac{1}{n} \sum_{t = 1}^{n} \loss\!\left(\tilde{f}_{i_k}(X_{t}),\goodfun(X_{t})\right)$.
In particular, note that on the event $E_0$, 
$\lim\limits_{k \to \infty} \limsup\limits_{n \to \infty} \beta_{k,n} \leq \lim\limits_{k \to \infty} 2^{-k} = 0$.
Therefore, Lemma~\ref{lem:array-convergence} implies that, on the event $E_0$, 
there exists a sequence $\{j_n\}_{n \in \nats}$ in $\nats$ with $j_n \leq k_n$ for all $n \in \nats$, 
such that $\lim\limits_{n \to \infty} \beta_{j_n,n} = 0$.

Thus, on the event $E_0$, 
\begin{multline*}
\limsup_{n \to \infty} \min_{1 \leq i \leq n^{1/4}} \frac{1}{n} \sum_{t = 1}^{n} \loss\!\left(\tilde{f}_{i}(X_{t}),\goodfun(X_{t})\right) 
\leq \limsup_{n \to \infty} \min_{1 \leq k \leq k_n} \frac{1}{n} \sum_{t = 1}^{n} \loss\!\left(\tilde{f}_{i_k}(X_{t}),\goodfun(X_{t})\right) 
\\ \leq \limsup_{n \to \infty} \frac{1}{n} \sum_{t = 1}^{n} \loss\!\left(\tilde{f}_{i_{j_{n}}}(X_{t}),\goodfun(X_{t})\right) 
= \limsup_{n \to \infty} \beta_{j_n,n} = 0.
\end{multline*}

Altogether, we have that on the event $E_0 \cap E_1$ of probability one (by the union bound), 
\begin{align*}
\limsup_{n \to \infty} \frac{1}{n} \sum_{t=1}^{n} & \left( \loss\!\left(\hat{f}_{t}(X_{1:(t-1)},Y_{1:(t-1)},X_t),Y_t\right) - \loss\!\left( \goodfun(X_t), Y_t \right) \right) 
\\ & \leq \limsup_{n \to \infty} \beta_{j_n,n} + 19\, n^{-1/4} \sqrt{\log(n)} + \frac{\hat{n}}{n} = 0.
\end{align*}
\end{proof}

\section{Open Problems}
\label{sec:open-problem}

This result raises a number of further questions on the subject of universal consistency for online learners with general families of processes  $(\ProcX,\ProcY)$.
One obvious question is whether the above result can be extended beyond metric losses.  
For instance, it is clearly desirable to at least extend the result to the problem of regression with the squared loss.
Another obvious question is whether the learning strategy itself can be simplified, or the result established for other more-familiar learning strategies, 
such as partition-based learning rules, $k$-nearest neighbor predictors (when $\X$ is finite-dimensional), or other local averaging predictors, 
with appropriate use of online regret arguments only in the selection of model complexity 
(e.g., the partition resolution in partition estimates, or choice of $k$ in $k$-nearest neighbors), 
rather than in the selection of the entire function as in the strategy used here.
This would extend such results which have been established for these learning rules under 
stationary ergodic processes \citep*{gyorfi:02a,gyorfi:12}.

{\vskip \pspace}A more formal question in this context is whether there is a \emph{largest} set $\ProcSet_{\X}$ of processes $\ProcX$ such that 
there exists an online learning rule that is strongly universally consistent for the family $\{ (\ProcX,\ProcY) : \ProcX \in \ProcSet_{\X} \}$.
More concretely, \citep*{hanneke:21} defines a set $\SUOL$ (strong universal online learning), 
the set of all processes $\ProcX$ such that there exists an online learning rule 
that is universally consistent for the set of all processes $(\ProcX,\ProcY)$ such that $Y_t$ is a deterministic function of $X_t$: 
that is, there exists a measurable $\target : \X \to \Y$ for which $Y_t = \target(X_t)$. 
\citet*{hanneke:21,hanneke:21c} has posed the question of whether there exists 
an online learning rule that is strongly universally consistent for the family of all $(\ProcX,\ProcY)$ with $\ProcX \in \SUOL$ and $Y_t$ a deterministic function of $X_t$
\citep*{hanneke:21c}, known as an \emph{optimistically} universal online learning rule \citep*{hanneke:21,hanneke:21c}.  
Very recently, \citet*{blanchard:22} has proposed a positive solution to this problem.
In the context of the present work, we may instead pose a stronger version of this question:

\begin{problem}
\label{prob:suol-dependent-Y}
Does there exist an online learning rule that is strongly universally consistent for the family $\{ (\ProcX,\ProcY) : \ProcX \in \SUOL \}$?
\end{problem}

\noindent We note that a positive resolution of this question would be strictly stronger than the open problem of \citep*{hanneke:21c} regarding the question of the existence of optimistically universal online learners, to which a solution has recently been proposed by \citet*{blanchard:22}.
Moreover, if the above question is answered positively, then it would identify the \emph{largest possible} set of processes $\ProcX$ for which, without any restrictions on $\ProcY$, there is an algorithm that is universally consistent for the family.\\

\bibliography{learning}

\end{document}